\documentclass[letterpaper, 10 pt, conference, nofonttune]{ieeeconf}
\IEEEoverridecommandlockouts
\usepackage{microtype}

\usepackage{cite}
\usepackage[bookmarks=true]{hyperref}
\usepackage{color}
\usepackage{times}
\usepackage{epsfig}
\usepackage{epstopdf}
\usepackage{graphicx}
\usepackage{amsmath}
\usepackage{amssymb}
\usepackage{bbm}
\usepackage[ruled,vlined,linesnumbered]{algorithm2e}
\usepackage{graphicx}
\usepackage{verbatim}
\usepackage{booktabs}
\usepackage{multirow}
\usepackage{makecell}
\usepackage{subfigure}
\usepackage{colortbl,booktabs}
\usepackage{lscape}
\usepackage{url}
\usepackage{makecell}
\usepackage{xcolor,colortbl}

\newcommand{\Real}{\mathbb{R}}

\newcommand{\www}{\mathbf{w}}
\newcommand{\uuu}{\mathbf{u}}
\newcommand{\TT}{\mathbf{O}}
\newcommand{\ttt}{\mathbf{o}}
\newcommand{\XX}{\mathbf{X}}
\newcommand{\xxx}{\mathbf{x}}
\newcommand{\II}{\mathbf{I}}

\newcommand{\vneg}{\vspace{0pt}}

\title{\LARGE \bf Simultaneous View and Feature Selection for Collaborative \\ Multi-Robot Perception}

\author{
    Brian Reily and Hao Zhang%
    \thanks{Brian Reily and Hao Zhang are with the Human-Centered Robotics Lab 
    in the Department of Computer Science
    at the Colorado School of Mines, Golden, CO, 80401.
    Email: \{breily, hzhang\}@mines.edu.}%
}

\begin{document}

\newtheorem{definition}{Definition}
\newtheorem{theorem}{\textbf{Theorem}}% [section]
\newtheorem{lemma}{\textbf{Lemma}}% [section]
\newtheorem{proposition}{Proposition}% [section]
\newtheorem{property}{Property}% [section]
\newtheorem{observation}{Observation}% [section]
\newtheorem{corollary}{\textbf{Corollary}}% [section]

\maketitle
\thispagestyle{empty}
\pagestyle{empty}

\begin{abstract}

Collaborative multi-robot perception provides multiple views of an environment, offering varying perspectives to collaboratively understand the environment even when individual robots have poor points of view or when occlusions are caused by obstacles. These multiple observations must be intelligently fused for accurate recognition, and relevant observations need to be selected in order to allow unnecessary robots to continue on to observe other targets. This research problem has not been well studied in the literature yet. In this paper, we propose a novel approach to collaborative multi-robot perception that simultaneously integrates view selection, feature selection, and object recognition into a unified regularized optimization formulation, which uses sparsity-inducing norms to identify the robots with the most representative views and the modalities with the most discriminative features. As our optimization formulation is hard to solve due to the introduced non-smooth norms, we implement a new iterative optimization algorithm, which is guaranteed to converge to the optimal solution. We evaluate our approach through a case-study in simulation and on a physical multi-robot system. Experimental results demonstrate that our approach enables effective collaborative perception through accurate object recognition and effective view and feature selection.

\end{abstract}

\section{Introduction}

Collaborative multi-robot perception enables a group of robots
to combine their individual observations to collectively gain a unified
understanding of the environment \cite{schmickl2006collective}.
Multi-robot systems can
provide multiple observations of objects, 
and usually enable views of the objects from different points of view.
In real-world scenarios, 
such as in disaster response and search and rescue applications \cite{baxter2007multi,correll2009multirobot},
individual robots can have their views occluded or sometimes
fully obstructed by obstacles and other environmental factors.
Collaborative perception can rely on the observations that contain
the most representative views of objects
obtained from multiple robots to effectively and
collaboratively perceive the objects despite the limitations of
individual robots.

\begin{figure}[t]
    \centering
    \includegraphics[width=0.48\textwidth]{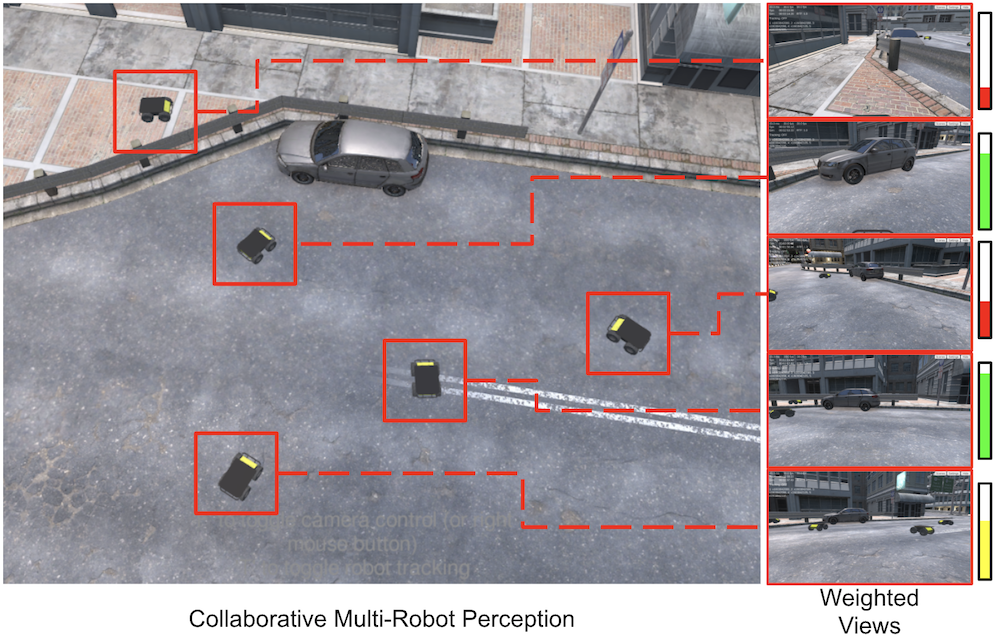}
    \vneg
    \caption{When a group of robots observe a target object in the environment,
    some robots may have partially occluded or completely obstructed views.
    Selecting the most representative views improves object recognition performance 
    and allow robots with less representative views to continue on to observe other
    targets. 
    }  
    \label{fig:overview}
\end{figure}

We address the collaborative multi-robot perception problem of recognizing an object 
from multiple views provided by a collection of robots operating around the object in the environment.
This perception problem becomes collaborative as the observations from multiple robots equipped with a variety of sensors
need to be combined into a unified understanding of the object.
Effectively fusing multi-modal observations from multiple robots relies on the
identification of the most informative views and modalities.
It not only increases the accuracy of object recognition, but also
enables robots to be distributed where they are most useful. 
For example,
robots behind obstacles can be assigned elsewhere, 
or robots equipped with LiDaR sensors can be assigned to dark areas while
those with only visual cameras can be used in an environment with
good lighting conditions.

% paragraph 3: difficulty, previous approaches

Previous research has investigated both recognition from multiple views
and discriminative sensor and view selection.
Multi-view recognition methods were developed to recognize objects
\cite{thomas2006towards}, gait \cite{kusakunniran2010support},
and actions \cite{spurlock2016dynamic} from multiple points of view.
%with additional approaches focusing on reconstructing 3D
%models from multiple 2D views \cite{genova2017learning}.
The previous approaches are generally based on 
visual features manually engineered for their specific applications.
Sensor and view selection has been studied in different ways,
e.g. based on sparse coding \cite{krause2012near} and uncertainty
reduction \cite{joneidi2017dynamic},
and applied to sensor placement and path planning in order to maximize
observations.
However, view selection and feature modality selection are
treated as two independent procedures.
The key research problem of how to enable both in a unified approach 
has not been well addressed,
especially in the area of collaborative perception for object recognition.

%but has not been applied to the task of determining relevant views
%among a group of robots.

% paragraph 4: our approach
In this paper, we introduce a novel approach to collaborative
multi-robot perception that enables object recognition from 
multiple views of a group of robots 
while simultaneously enabling selection of 
discriminative views and features.
We propose a formulation based on regularized optimization that
identifies the combination of views that best represents
and recognizes objects,
providing a unified framework that incorporates recognition,
view selection, and feature selection.
Through the introduction of structured sparsity-inducing norms as
regularization terms during the optimization,
we identify both the robots with the most representative views, as well
as the most discriminative feature modalities.
Due to non-smooth norms and parameter dependence, 
our formulated regularized optimization is hard to solve.
We implement a new iterative solver, which we prove is theoretically guaranteed to
converge to the optimal solution.
We perform extensive evaluation utilizing a high-fidelity
robotics simulator and a physical multi-robot system.
Experimental results show that our approach
achieves both accurate recognition as well as effective view selection,
and shows the effectiveness of our proposed regularization
terms.

% paragraph 5: contributions
This paper has two key contributions:
\begin{itemize}
    \item First, we introduce a principled approach to 
    collaborative multi-robot perception which unifies multi-view 
    object recognition, view selection, and
    feature selection into a single formulation based on the mathematical
    framework of regularized optimization.
    \item Second, we implement a new iterative solution algorithm to solve
    the proposed formulation, which has a theoretical guarantee 
    to converge to the optimal solution.
\end{itemize}

%% paragraph 6: structure
%The remainder of this paper is structured as follows.
%We discuss related work in Section \ref{sec:related}.
%We introduce our approach for joint view selection and multi-modal
%visual fusion for object recognition in Section \ref{sec:approach}.
%In Section \ref{sec:results}, we present the experimental results of our approach.
%Finally, we conclude the paper in Section \ref{sec:conclusion}.

\section{Related Work} \label{sec:related}

%\cite{pluim2003mutual}: Registration of images based on mutual information.

%Collaborative perception with multi-view recognition and selection
%involves multiple research areas.
%We discuss current work in multi-robot perception, with a focus on
%active perception and the control
%of robots based upon their observations.
%We review approaches for object recognition from multiple
%views, and discuss techniques for selection of discriminative 
%and informative features and sensors.

\subsection{Multi-Robot Perception}

Our proposed work closely aligns with the research problem that
optimally places and controls robots based on their observations.
Many methods attempt only the placement of fixed sensors to maximize
observations \cite{parker1997cooperative,ma2010dynamic,domingo2016sensor}.
However, many others work to coordinate multiple robots to maximize
sensor coverage, such as with power- and capability-limited robots
used in swarms \cite{singh2009efficient},
or to identify corresponding objects between 
observations \cite{gao2020regularized}.
Similar approaches work by dividing swarms into subgroups to
maximize the monitoring of multiple areas of
interest \cite{inacio2018persistent}, deploying \cite{liu2018optimal} 
or dividing \cite{reily2020representing} heterogeneous teams
of robots based on their sensing capabilities,
or attempting to merge sensor observations to provide
`collective perception' \cite{schmickl2006collective}.
In contrast to these approaches, we formulate our method to
integrate object recognition as well as identifying the individual
robots with the most discriminative observations.
By identifying the robots with the most representative views as they are
observing a target and the most discriminative sensing modalities,
we allow the remainder of the robots to continue
on to other tasks.

Another relevant research area is active perception,
with the objective to adjust sensor positions and settings
to obtain optimal views of objects \cite{chen2011active}.
In multi-robot systems, active perception has been
applied to navigation planning and target tracking.
Navigation planning approaches apply active perception as a reward
or constraint, where paths that involve the most observations
of areas are rewarded \cite{chiu2014constrained,best2016multi}.
Target tracking approaches apply active perception to the goal
of observing a single or set of targets as opposed to observing
an area.
Approaches have been based on Kalman filters \cite{dietl2001cooperative},
scheduling algorithms \cite{shi2015probabilistic}, and entropy
measures \cite{hausman2015cooperative}.

In computer vision applications, multi-view perception has been applied to a
variety of recognition tasks, utilizing fixed sensors or cameras.
Multi-view recognition has been applied to object recognition
\cite{thomas2006towards,wu2016multi}, gait recognition
\cite{kusakunniran2010support}, and action recognition
\cite{spurlock2016dynamic}.
Multiple views typically increase recognition accuracy, but can also
incur a high computation cost.
Because of this challenge, approaches have been developed to narrow down
the number of views and identify the views that can best represent
objects \cite{mokhtarian2000automatic,moreira2006best}.
Several methods were also designed to select the 2D views that best reconstruct a 3D model
\cite{laga2010semantics,mendez2016next,genova2017learning,wang2017view}.
Methods were also implemented to identify objects
from multiple cameras, including search algorithms \cite{wang2016constructing},
convolutional neural networks \cite{kim2017category},
and semantics \cite{guerin2018semantically}.

\subsection{Feature Selection}

Similar to our goal of identifying the most discriminative views among
a group of robots, many approaches have been developed to select features
\cite{zhang2019feature} and sensors that minimize redundancy
while maximizing information.
Feature selection approaches identify the most discriminative
features, either as a problem of dimensionality reduction or
to identify the certain features that correspond to specific
situations.
This has been done by selecting features that reduce uncertainty in
a graphical model \cite{krause2012near}, utilizing sparsity to
identify feature modalities that best describe an environment
\cite{han2017sral}, learning weights of different features
\cite{tang2018consensus}, or evaluating information content to avoid
selecting redundant features \cite{peng2005feature}.

The selection of discriminative sensors in sensor networks has been
studied from a variety of directions \cite{gravina2017multi}.
For target tracking, sensors can be chosen based on their
capabilities of location \cite{wang2018sensor,xiao2017divide}
or to best provide coverage of an entire area \cite{zhang2015distributed}.
Selection methods have also been introduced based on task allocation
algorithms \cite{tkach2017multi}, sparsity
\cite{joneidi2017dynamic,liu2015sparsity},
cross entropy \cite{zhang2018spatial}, and even randomized
algorithms \cite{bopardikar2019randomized}.

The previous multi-view recognition and feature selection approaches 
generally treat view section and feature modality selection as two separate and independent procedures.
The problem of how to address both in a unified approach
has not been well addressed,
especially in the area of collaborative perception for object recognition.

\section{Our Proposed Approach}
\label{sec:approach}

In this section, we introduce our approach to collaborative
perception using an optimization-based formulation that
unifies recognition and the selection of discriminative views
and feature modalities.

\emph{Notation}.
In this paper, we denote matrices using boldface uppercase letters and
vectors using boldface lowercase letters.
For a matrix $\mathbf{M} = \{ m_{ij} \}$, we denote its $i$-th row
as $\mathbf{m}^i$ and its $j$-th column as $\mathbf{m}_j$.
The $\ell_1$-norm of a vector $\mathbf{v} \in \Real^n$ is
$\| \mathbf{v} \|_1 = \sum_{i=1}^{n} | v_i |$ and the $\ell_2$-norm is
$\| \mathbf{v} \|_2 = \sqrt{\mathbf{v}^\top \mathbf{v}}$.

\subsection{Problem Formulation}

We define the problem of collaborative multi-robot perception
as combining the available views from a group of robots in order to
identify an object.
Our formulation learns a linear combination of views to
represent a given object, fusing the visual features
from each to identify the object that they best represent,
based on the value of the final objective function.

Formally, we define the observations from $n$ robots as
$\XX = [ \xxx^1 ; \dots ; \xxx^n ] \in \Real^{n \times d}$,
where $\xxx^i \in \Real^{d}$ is the feature vector denoting the observations
of the $i$-th robot.
As each robot can utilize multiple sensors or describe observations
as multiple forms of feature representations, each vector
$\xxx^i$ contains representations from $m$ modalities, where
$d = \sum_{i=1}^{m} d_i$.
We denote $p$ categories of objects as
%%%% \TT makes a bold O now, changed notation and didn't want to change the text
%%%% Same for \ttt (o)
$\TT = [ \ttt_1, \dots, \ttt_p ] \in \Real^{d \times p}$, where
$\ttt_j \in \Real^{d}$ is a feature vector representing the $j$-th object.
Objects are encoded with the same feature
types that are able to be utilized by the robots, defined from a single
view of an object (e.g., if robots represent their observations with a
histogram of RGB values, then each object is encoded with the same type 
of histogram formed from a representative view of the object).
Then, we formulate object recognition in our collaborative multi-robot perception
approach with the following loss function:
\begin{align}
\min_{\www} & \| \XX^\top \www - \ttt_j \|_2^2
\label{eq:loss}
\end{align}
where we approximate each object $\ttt_j$ for $j = 1, \dots, p$
through a linear combination of the available views from the multi-robot system.
The views are weighted by the weight vector $\www \in \Real^{n}$,
where $w_i$ represents the importance of the $i$-th robot's view in
approximating the object.

\subsection{Learning Discriminative Views and Modalities}

When a group of robots are observing an object, 
a subset of robots will have more representative views of the object, 
and certain feature
modalities will be much more informative than others.
We introduce regularization terms in our formulation to identify the robots
with the most discriminative views and the most relevant features.

First, we learn the most informative views, both to rely on them
for accurate recognition and to identify the robots with the best views,
so that the remaining robots can continue on in the environment to perform
other tasks.
To do this, we introduce $\| \www \|_1$, utilizing the $\ell_1$-norm on
the weight vector $\www$.
This regularization term induces sparsity in this weight vector, forcing
most values to $0$ or very small values near $0$ and limiting high
weights to the most discriminative views.
With this, our formulation is
\begin{align}
\min_{\www} & \| \XX^\top \www - \ttt \|_2^2 + \lambda \| \www \|_1
\label{eq:loss_with_l1}
\end{align}
where $\lambda$ is a hyperparameter controlling the importance of this
sparsity-inducing norm.

Second, we learn the most discriminative
feature modalities, not only to further improve recognition accuracy,
but also to identify the sensing modalities that are most useful in the
environment - i.e., due to varying lighting conditions and views between
known objects and observations, certain modalities and representations
could be significantly more informative.
To achieve this,
we introduce a weight vector $\uuu \in \Real^{d}$, which
denotes the importance of each feature modality to the view weight vector:
\begin{align}
\| \XX \uuu - \www \|_2^2
\label{eq:modal_loss}
\end{align}
This loss-like regularization term influences the learning of $\www$, 
by incorporating the weighting of individual features.
It also enables us to identify the importance of each feature
modality, as $\uuu = [ \uuu^1, \dots, \uuu^m ]$ where
$\uuu^i \in \Real^{d_i}$ specifically represents the importance
of the $i$-th feature modality to the view selection weights in $\www$.

Then, to identify the discriminative feature
modalities, we introduce a group $\ell_1$-norm applied to  $\uuu$ as the regularization, 
which is termed the
\emph{modality} norm:
\begin{equation}
\| \uuu \|_M= \sum_{i=1}^{m} \| \uuu^i \|_2
\end{equation}
This regularization term utilizes the $\ell_2$-norm within feature modalities
and the $\ell_1$-norm between modalities.
The $\ell_2$-norm causes weights within a modality to become similar, while
the $\ell_1$-norm induces sparsity between them, causing only the most
discriminative modalities to have non-zero weights.

With these two regularization terms to identify discriminative views and feature modalities, respectively, the proposed final problem formulation becomes:
\begin{equation}
\min_{\www,\uuu}  \| \XX^\top \www - \ttt \|_2^2 +
\lambda_1 \| \XX \uuu - \www \|_2^2 + \lambda_2 \| \www \|_1 +
\lambda_3 \| \uuu \|_M
\label{eq:full}
\end{equation}
where $\lambda_1$, $\lambda_2$, and $\lambda_3$ are hyperparameters that control
the importance of the regularization terms.

This optimization problem is calculated for each object.
The recognized label of a given object is the category that has the lowest value of the
objective function, since it is best represented by the multiple views.
This proposed formulation simultaneously integrates view and feature selection, as well as object
recognition into the unified regularized optimization framework,
without requiring a separate classifier for object recognition. 

\begin{algorithm}[tb]
\SetAlgoLined
\SetKwInOut{Input}{Input}
\SetKwInOut{Output}{Output}
\SetNlSty{textrm}{}{:}
\SetKwComment{tcc}{/*}{*/}

% \small

\Input{
$\XX = [ \xxx^1 ; \dots ; \xxx^N ] \in \Real^{n \times d}$,
$\TT = [ \ttt_1, \dots, \ttt_p ] \in \Real^{d \times p}$.
}
\Output{
$\ttt^*$ (the recognized object) and
$\www^*$ and $\uuu^*$ (weight vectors identifying the most
representative views and discriminative features).
}
\BlankLine

\ForEach{column vector $\ttt$ in $\TT$}{

Let $i = 1$. Initialize $\www$ by solving Eq. (\ref{eq:loss})
and $\uuu$ by then minimizing Eq. (\ref{eq:modal_loss}).

\Repeat{convergence}{
Calculate $\mathbf{D}^w (i + 1)$, where the diagonal is equal to
$\frac{1}{2 \www \left( i \right)}$.

Calculate $\mathbf{D}^u (i + 1)$, where the $j$-th diagonal block
is equal to $\frac{1}{2 \| \uuu^j \left( i \right) \|_2} \II_j$.

Calculate $\www \left( i + 1 \right)$ via
    Eq. (\ref{eq:solve_w}).

Calculate $\uuu \left( i + 1 \right)$ via
    Eq. (\ref{eq:solve_u}).

$i = i + 1$.
}

Compute the value of Eq. (\ref{eq:full}) using $\www \left( i \right)$ and
$\uuu \left( i \right)$.

\uIf{objective value is lowest}{
$\ttt^* = \ttt$, $\www^* = \www \left( i \right)$, $\uuu^* = \uuu \left( i \right)$
}
}

\Return the recognized object $\ttt^*$ and
the associated $\www^*$ and $\uuu^*$.

\caption{Our iterative algorithm to solve the formulated regularized optimization
problem in Eq. (\ref{eq:full}).}
\label{alg:solution}
\end{algorithm}

\subsection{Optimization Algorithm}

Due to the non-smooth norms utilized to identify the most representative
views and most discriminative features and the interdependence of the
$\www$ and $\uuu$ weight vectors, our formulation is hard to solve.
We introduce an iterative algorithm to solve the formulation in
Eq. (\ref{eq:full}), which alternately solves for $\www$ and $\uuu$ at
each iteration until convergence.
We show that our approach is theoretically guaranteed to converge
to optimal values for the weight vectors.

First, we solve $\www$ by taking the derivative of the objective
function with respect to $\www$ and setting it to $0$:
\begin{align}
2 \XX \XX^\top \www - 2 \XX \ttt - 2 \lambda_1 \XX \uuu + 2 \lambda_1 \II \www + \lambda_2 \mathbf{D}^w \www = \mathbf{0}
\end{align}
Here, $\II$ is the identity matrix and 
$\mathbf{D}^w \in \Real^{n \times n}$ is a diagonal matrix
where the $i$-th diagonal element is $\frac{1}{2 w_i}$, corresponding
to the partial derivative of $\| \www \|_1$.
After rearrangement, we see that $\www$ is updated by
\begin{align}
\www = \left( \XX \XX^\top + \lambda_1 \II + \frac{\lambda_2}{2} \mathbf{D}^w\right)^{-1} \XX \left( \ttt + \lambda_1 \uuu \right)
\label{eq:solve_w}
\end{align}

Second, we now solve for $\uuu$ by taking the derivative of the
objective function with respect to $\uuu$ and setting it equal
to $0$:
\begin{align}
2 \lambda_1 \XX^\top \XX \uuu - 2 \lambda_1 \XX^\top \www + \lambda_3 \mathbf{D}^u \uuu = \mathbf{0}
\end{align}
where $\mathbf{D}^u$ is a block diagonal matrix corresponding
to the partial derivative of $\| \uuu \|_M$, with its $i$-th block
as $\frac{1}{2 \| \uuu^i \|_2} \II$, where $\uuu^i$ is the section
of the weight vector $\uuu$ corresponding to the $i$-th modality
and $\II$ is an identity matrix with appropriate dimensions for 
$\uuu^i$.
We rearrange it to see that $\uuu$ is updated by
\begin{align}
\uuu = \left( \lambda_1 \XX^\top \XX + \frac{\lambda_3}{2} \mathbf{D}^u \right)^{-1} \lambda_1 \XX^\top \www
\label{eq:solve_u}
\end{align}

In each update step, we assume that the other weight vector is 
fixed (e.g., in Eq. (\ref{eq:solve_w}) we treat $\uuu$ as fixed from 
the previous iteration).
We alternate these update steps until the values of each variable converge.
Algorithm \ref{alg:solution} details our implemented optimization solver.
In the following, we show that our proposed algorithm theoretically converges to the optimal solution.
% \todo{One reviewer asked about invertibility for matrices in
% Eq 7 and 9 - since these are based on raw sensor observations and thus 
% extremely unlikely to be singular, is this worth commenting on?}

\begin{theorem}\label{thm1}
The inner loop of Algorithm 1 is guaranteed to converge to the optimal
solution to the formulated regularized optimization problem in Eq. (\ref{eq:full}).
\end{theorem}

\begin{proof}
See Appendix.
\end{proof}

The time complexity of Algorithm \ref{alg:solution} is dominated
by Steps (6) and (7) in each iteration.
Steps (4) and (5) are trivial, being only linear in complexity.
Steps (6) and (7) each require a matrix inverse, which could be solved as
a system of linear equations as opposed to explicitly finding the inverse, and
so respectively require $\mathcal{O} \left( n^2 \right)$ and
$\mathcal{O} \left( d^2 \right)$ to solve.
As typically $d \ggg n$, our approach's complexity is bounded by
the dimensionality of features chosen, and not the number of 
robots.

\section{Experimental Results}\label{sec:results} 

To effectively assess the performance of our proposed collaborative
multi-robot recognition approach, we evaluated through a high-fidelity 
robot simulator, and also executed our approach on physical robots.
We also discuss the effect
of the hyperparameters $\lambda_1$ and $\lambda_2$.

In both of our experiments, we utilize three feature modalities.
First, we utilize a color histogram, based on the RGB
values of each pixel.
Second, we utilize a Histogram of Oriented Gradients (HOG) \cite{dalal2005histograms}
to describe the shape of the observation.
Finally, we describe the textures in an observation using Local
Binary Pattern (LBP) features \cite{ojala1994performance}.
Each observation is described fully, i.e. the object to be recognized is
not defined by a bounding box from labeling or detection.
This enables us to evaluate our approach as it would actually be
used in the real world, where object detection is unavailable on
computing-limited robots.

\begin{figure*}
    \centering
    \subfigure{
        \centering
        \includegraphics[height=1.6in]{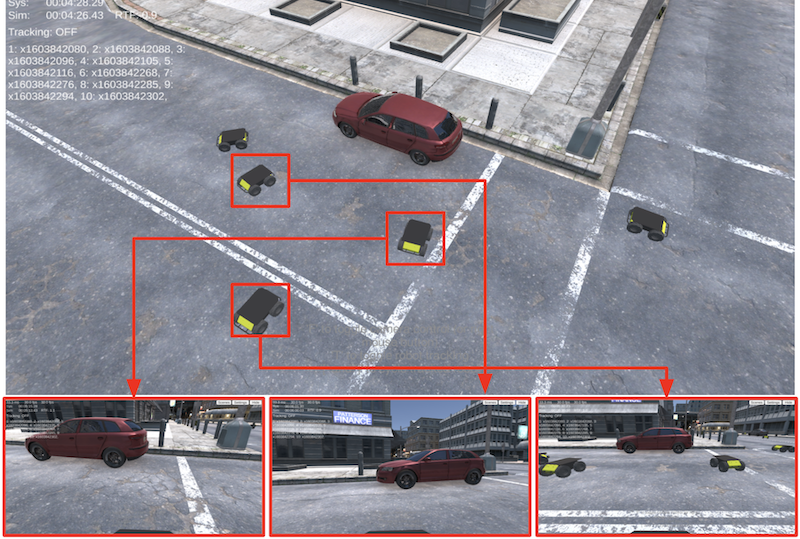}
        \label{fig:dcist_a}
    }%
    \subfigure{
        \centering
        \includegraphics[height=1.6in]{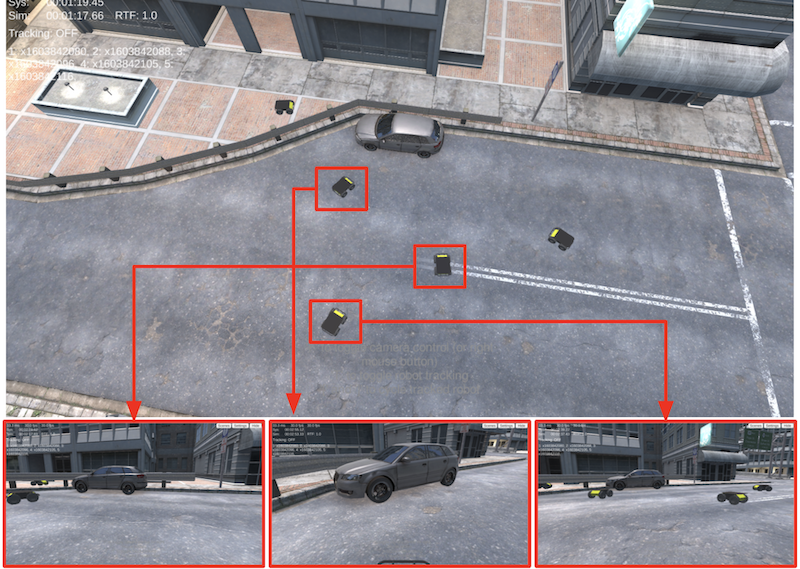}
        \label{fig:dcist_b}
    }%
    \subfigure{
        \centering
        \includegraphics[height=1.6in]{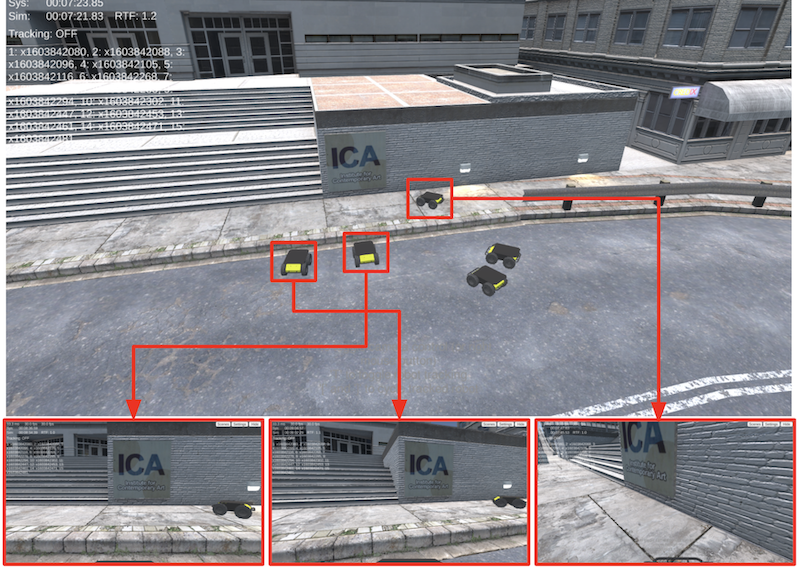}
        \label{fig:dcist_c}
    }%
    \vneg
    \caption{We conducted our case-study in a high-fidelity robot
    simulator with three objects -- the red car, the gray car, and the
    sign seen in these overhead scenes.
    The three best views from the multi-Husky system 
    as ranked by our approach are shown beneath each overhead scene. 
    }
    \label{fig:dcist}
\end{figure*}

\subsection{Evaluation in Multi-Robot Simulation}

We first performed a case-study evaluation in a high-fidelity simulator,
which allowed for the simulation of real robots and sensors through
a ROS interface.
We chose three different objects: two cars, a red one and a gray one, 
in which the known reference
view was from the side, and a large sign with the text `ICA',
with the known view being straight on.
We utilized a system of 5 Husky robots equipped with RGB cameras.

Qualitative results from this can be seen in Figure \ref{fig:dcist},
which shows the top three views from each different scene as ranked
by our approach.
We can see that our approach is able to select the Husky robots which have
qualitatively good views of each object, allowing the other robots
to continue on to other tasks.
These views seem straightforward for a human to select when viewing the scene
after the fact, but when viewed from overhead clearly show that our approach
is selecting robots with good angles and unobstructed views.
In a search and rescue application, our approach can clearly identify when a 
multi-robot system has located a target, allowing the remainder of the
robots to continue on to search for further targets.

\subsection{Evaluation on a Physical Multi-Robot System}

\begin{figure*}
    \centering
     \subfigure[Multi-Robot Views]{
        \centering
        \includegraphics[width=0.31\textwidth]{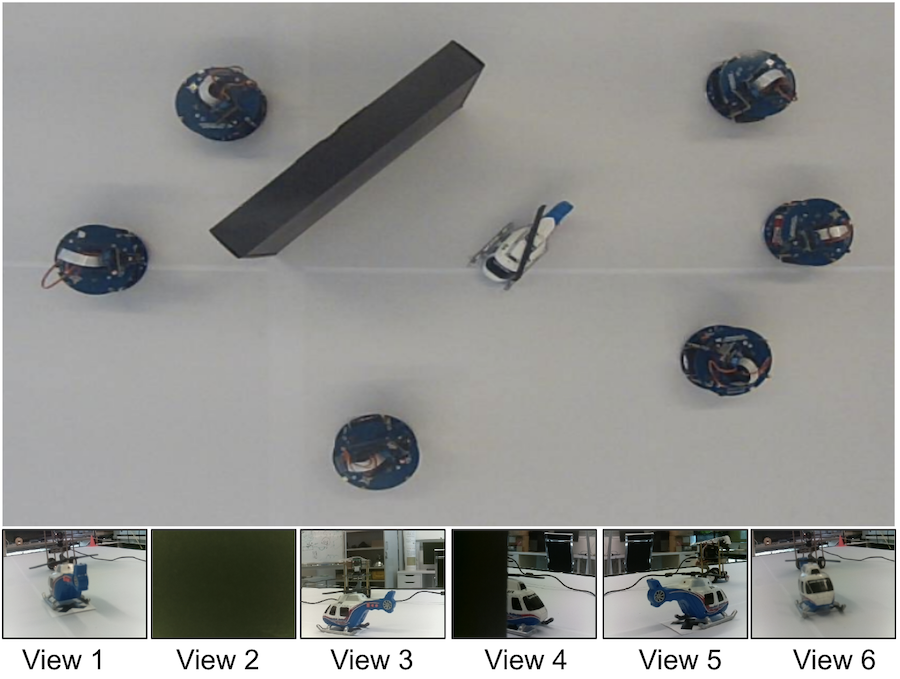}
        \label{fig:sparsity_setup}
    }%
    \subfigure[Recognition Accuracy]{
        \centering
        \includegraphics[width=0.31\textwidth]{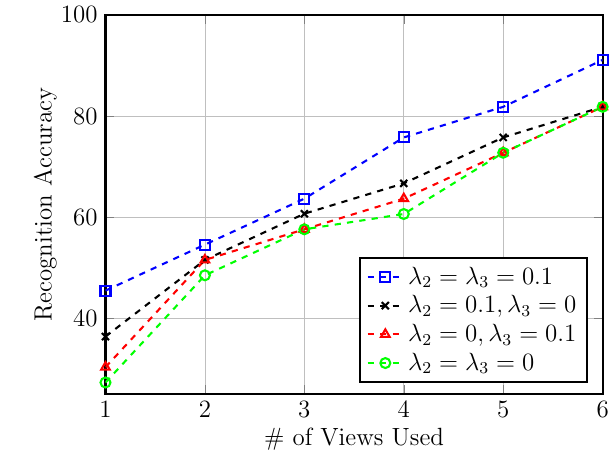}
        \label{fig:table_results}
    }%
    \subfigure[Mutual Information]{
        \centering
        \includegraphics[width=0.32\textwidth]{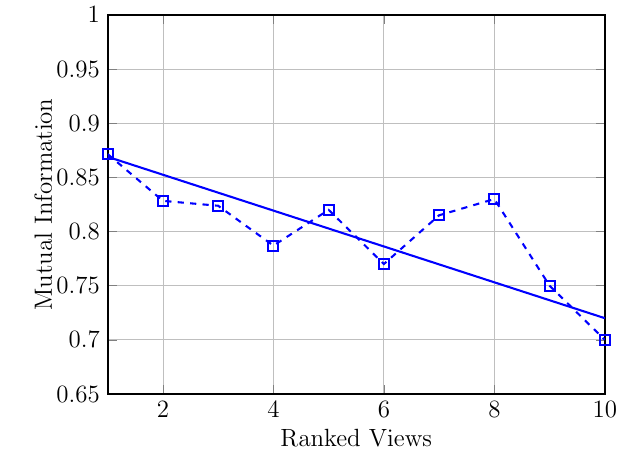}
        \label{fig:dataset_table_mi}
    }%
    \vneg
    \caption{Quantitative results obtained our collaborative perception approach.
    Figure \ref{fig:table_results} shows the recognition accuracy on the
    physical multi-robot system. Figure \ref{fig:dataset_table_mi}
    shows the correlation between the views selected by our approach
    and the mutual information they share with the ground truth objects.
    }
    \label{fig:results_table}
\end{figure*}

Next, we evaluated our approach on a physical multi-robot system.
This system consists of an overhead camera for robot tracking
and six robots, each using a Raspberry Pi 3+
for on-board computing and equipped with an RGB camera.
The set of objects consisted of eight objects of varying size,
shape, and color.
For each object, multiple sets of data were collected with a
varying amount of obstructions by obstacles.
Robots were positioned surrounding the object, with views that
were either unobstructed, partially obstructed to different
degrees, or fully obstructed.
An overhead view of this setup can be seen in 
Figure \ref{fig:sparsity_setup}.

Quantitative results on recognition are reported in 
Figure \ref{fig:table_results}.
We evaluate four versions of our approach, by setting $\lambda_2$
and $\lambda_3$ to either $0$ or $0.1$ (for both a baseline analysis and 
an optimal parameter setting, described later in Section \ref{sec:disc}).
We evaluate our approach by selecting a random subset of $n$ views, from
a single view to all six views, and repeat this 10 times for each
value of $n$.
We observe consistently that when any number of views are used, setting
$\lambda_2 = \lambda_3 = 0$ performs the worse compared to any other
set of parameters.
Similarly, recognition accuracy is consistently at its highest when
$\lambda_2 = \lambda_3 = 0.1$, achieving the best recognition for
each set of $n$ views.
This demonstrates the value of our two introduced regularization
terms.
We also observe that when using only one of the sparsity-inducing norms
(i.e., setting one $\lambda$ parameter to $0.1$ and the other
to $0$) we see better performance when using $\| \www \|_1$,
forcing the use of the most representative robots.
This result indicates that relying on a better view provides
more accurate recognition than better features in random views.

Figure \ref{fig:dataset_table_mi} shows the amount of mutual information
between the ground truth object and input views, ordered by their
weights in $\www$.
We can see a clear trend line showing that 
views highly weighted by our approach also share a large amount of information
with the ground truth objects.

\begin{figure*}
    \centering
    \subfigure[Weight Vector $\www$]{
        \centering
        \includegraphics[height=1.65in]{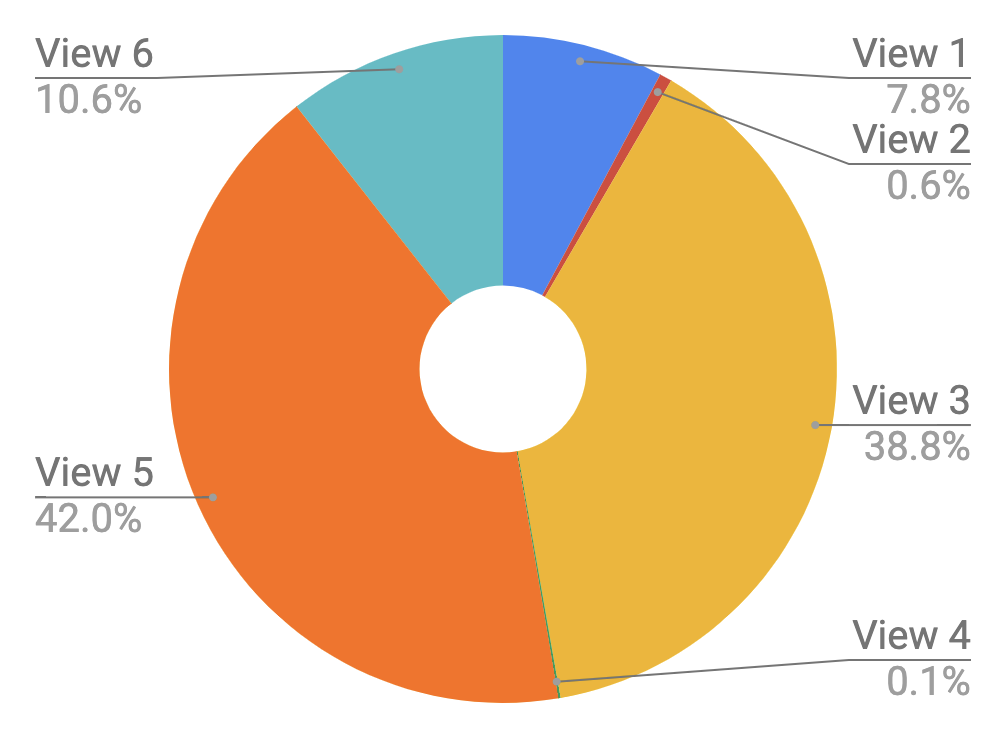}
        \label{fig:sparsity_w}
    }%
    \subfigure[Weight Vector $\uuu$]{
        \centering
        \includegraphics[height=1.65in]{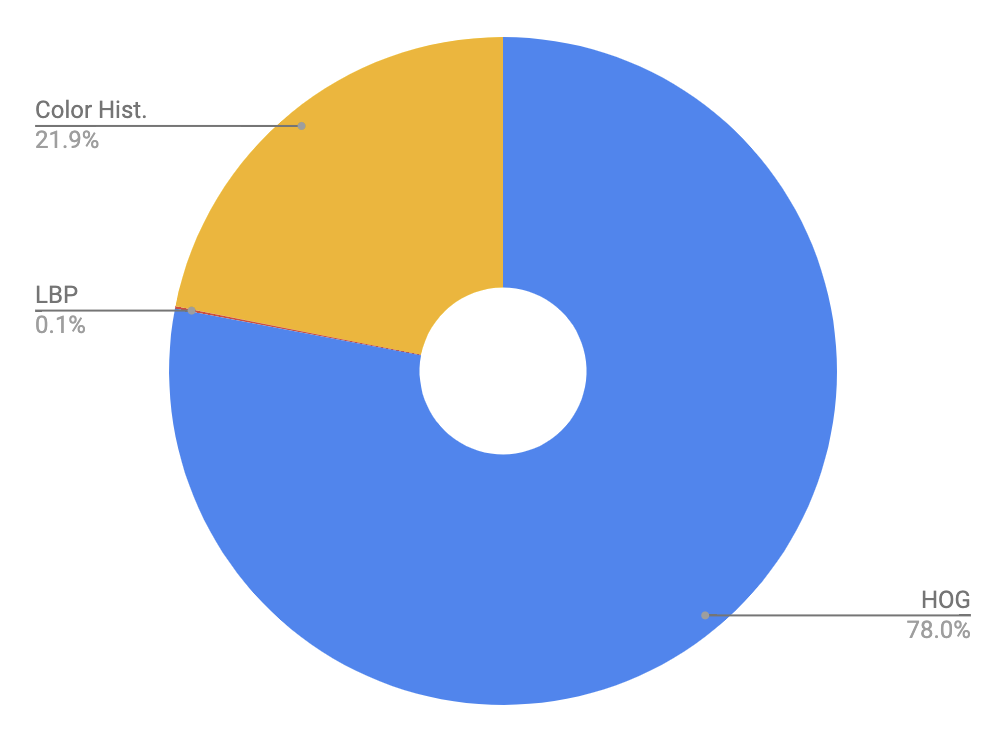}
        \label{fig:sparsity_u}
    }%
    \subfigure[$\lambda_2$ and $\lambda_3$]{
        \centering
        \includegraphics[height=1.65in]{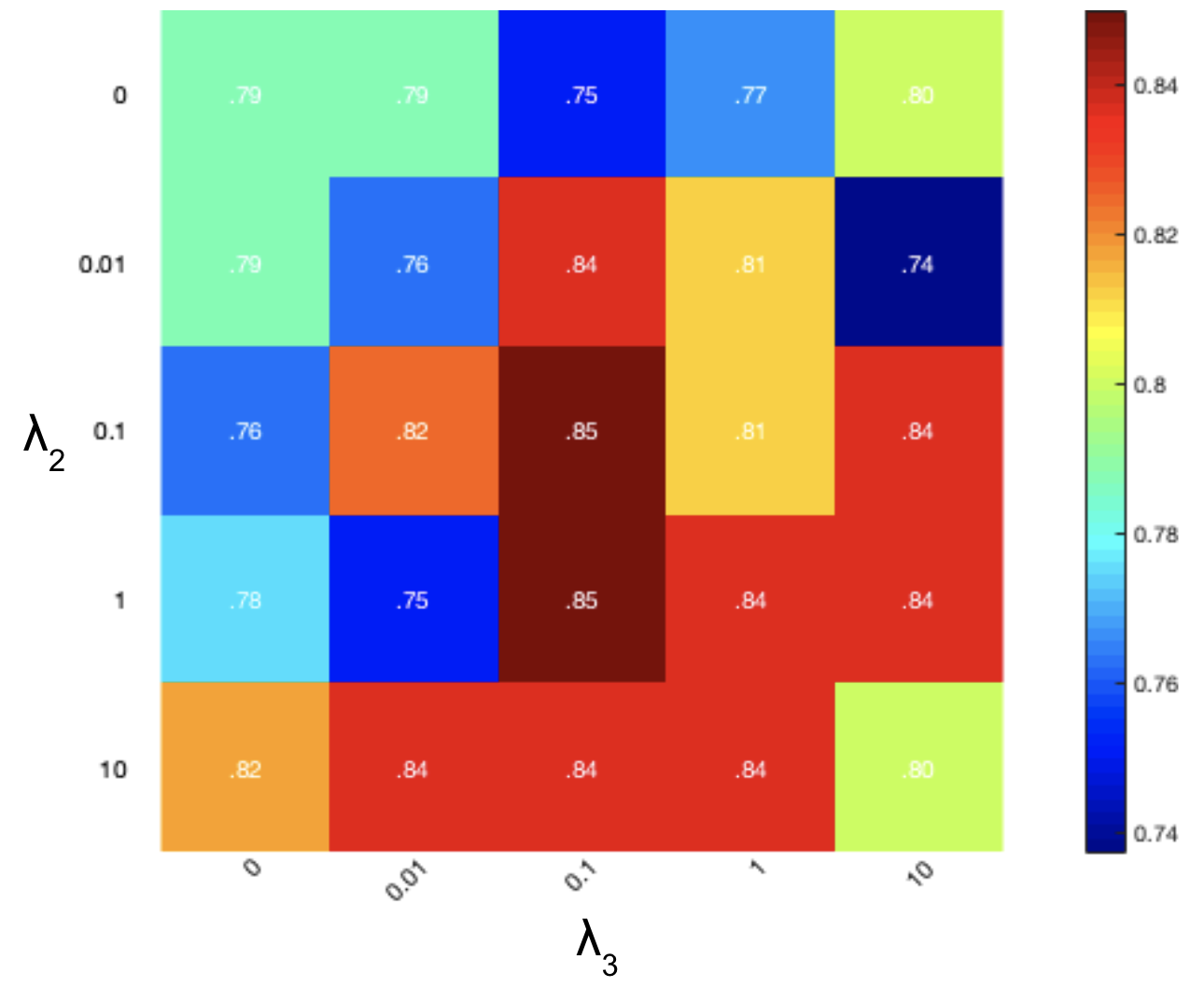}
        \label{fig:coil_l1_l2}
    }%
    \vneg
    \caption{These figures show the effect of our two introduced
    sparsity-based regularization terms to identify the most
    representative views and the most discriminative features.
    Figure \ref{fig:sparsity_setup} shows a testing iteration of
    the multi-robot system observing an object, with the six
    robot views.
    Figures \ref{fig:sparsity_w} and \ref{fig:sparsity_u} show
    the associated weight vectors indicating the weights assigned
    to various robots and features.
    Figure \ref{fig:coil_l1_l2} quantifies the effect of
    various hyperparameter values controlling the importance
    of the two regularization terms.
    }
    \label{fig:sparsity}
\end{figure*}

\subsection{Discussion}
\label{sec:disc}

\textbf{Discriminative View and Feature Modality Selection.}
We evaluate the performance of view and feature selection 
enabled by the two sparsity-inducing norms we introduce.
%As both of these norms are intended to produce sparsity in their
%associated vectors, we should see the majority of the weights in
%each be assigned to a small number of views or modalities.
%
Figures \ref{fig:sparsity_w} and \ref{fig:sparsity_u} show the
distribution of weights for the testing iteration seen in
Figure \ref{fig:sparsity_setup}, where six robots observe
the helicopter object from the physical robot evaluation.
In this test, an obstacle is completely obstructing one robot's view
and partially obstructing the view of a second.
The four remaining robots have unobstructed views of the object,
but from varying angles.
We see values of nearly $0$ assigned to the two 
obstructed views in $\www$,
with small values assigned to two other poor views of the target
object.
Over $80\%$ of weights are assigned to two views, identifying these as
the most representative of the target.
In $\uuu$, we observe that the total weight for the LBP features
is nearly $0$, with HOG features being the primary discriminative
modality for this setup.
Both figures demonstrate the effectiveness of the induced sparsity
in identifying a small number of representative views and discriminative
modalities.

\textbf{Hyperparameter Analysis.}
We also evaluate the effect of the hyperparameters $\lambda_2$ and
$\lambda_3$, which control the importance of these sparsity-inducing norms,
on the recognition accuracy of our approach.
Figure \ref{fig:coil_l1_l2} shows this accuracy as the values
of these parameters change.
Primarily, we observe that setting either of these values
to $0$ achieves low relative accuracy, showing the necessity
of these regularization terms.
Similarly, there is low performance when $\lambda_2 = \lambda_3 = 10$,
as this causes the loss functions to be ignored.
We see that the highest
accuracy comes for $\lambda_2 = 0.1$ or $\lambda_2 = 1$
and $\lambda_3 = 0.1$, indicating that balancing the weight of
these terms versus the loss function achieves the best performance.

\section{Conclusion}
\label{sec:conclusion}

Collaborative perception allows a multi-robot system to perceive an environment
from multiple perspectives,
which is able to use the robots that have the best views
to obtain an optimal understanding of the environment.
We introduce a novel approach to collaborative multi-robot perception
that simultaneously incorporates view selection, feature selection, 
and object recognition
into a unified regularized optimization formulation.
Sparsity-inducing norms are designed to achieve the identification of
the most representative views and features.
We perform extensive evaluation with a
case-study in a high-fidelity simulator and evaluation on 
a physical multi-robot 
system, with our experimental results demonstrating
both accurate object recognition as well as effective
view and feature selection.

\section*{Appendix: Proof of of Theorem 1}

This section provides a proof that Algorithm 1
converges to an optimal value for $\www$ and $\uuu$ for each template,
with the inner loop decreasing the value of the objective function in
Eq. (5) in each iteration.

First, we present the following lemma from \cite{nie2010efficient}:
\begin{lemma}\label{lemma1}
For any two vectors $\mathbf{v}$ and $\mathbf{\tilde{v}}$, the following inequality relation holds:
$\|\mathbf{\tilde{v}}\|_2 - \frac{\|\mathbf{\tilde{v}}\|_2^2}{2\|\mathbf{v}\|_2}
\leq
\|\mathbf{v}\|_2 - \frac{\|\mathbf{v}\|_2^2}{2\|\mathbf{v}\|_2}
$.
\end{lemma}

\begin{proof}
We have:
\begin{eqnarray}
\left(\Vert\widetilde{\mathbf{v}}\Vert_{2}-\Vert\mathbf{v}\Vert_{2}\right)^2 \geq 0
\end{eqnarray}
\begin{eqnarray}
\Vert\widetilde{\mathbf{v}}\Vert_{2}^{2} -  2\Vert\widetilde{\mathbf{v}}\Vert_{2}\Vert\mathbf{v}\Vert_{2} + \Vert\mathbf{v}\Vert_{2}^{2} \geq 0
\end{eqnarray}
\begin{eqnarray}
\Vert\widetilde{\mathbf{v}}\Vert_{2} - \frac{\Vert\widetilde{\mathbf{v}}_2^2}{2\Vert\mathbf{v}\Vert_2} \leq \frac{\Vert\mathbf{v}\Vert_2}{2}
\end{eqnarray}
\begin{eqnarray}
\Vert\widetilde{\textbf{v}}\Vert_{2} - \dfrac{\Vert\widetilde{\textbf{v}}\Vert_{2}^{2}}{2\Vert\textbf{v}\Vert_{2}} \leq  \Vert\textbf{v}\Vert_{2} - \dfrac{\Vert\textbf{v}\Vert_{2}^{2}}{2\Vert\textbf{v}\Vert_{2}}
\end{eqnarray}
\end{proof}

Similarly, we see that 
\begin{eqnarray}
\Vert\widetilde{\textbf{v}}\Vert_1 - 
\dfrac{\Vert\widetilde{\textbf{v}}\Vert_2^2}{2\Vert\textbf{v}\Vert_1} \leq
\Vert\textbf{v}\Vert_1 - \dfrac{\Vert\textbf{v}\Vert_2^2}{2\Vert\textbf{v}\Vert_1}
\end{eqnarray}

Using these, we prove Theorem 1 from the main paper.
First, according to Algorithm 1, $\www$ and $\uuu$ are updated by
\begin{eqnarray}
\min_{\www,\uuu}  \| \XX^\top \www - \ttt \|_2^2 +
\| \XX \uuu - \www \|_2^2 + \lambda_1 \| \www \|_1 +
\lambda_2 \| \uuu \|_M
\label{eq:full}
\end{eqnarray}

Defining $\mathcal{L} \left( i + 1 \right) = \| \XX^\top \www - \ttt \|_2^2 +\| \XX \uuu - \www \|_2^2$, we derive that 
\begin{eqnarray}
\mathcal{L} \left( i + 1 \right) + 
\lambda_1 \www^\top \left( i + 1 \right) \mathbf{D}^w \left( i + 1 \right) \www \left( i + 1 \right) + \notag \\
\lambda_2 \uuu^\top \left( i + 1 \right) \mathbf{D}^u \left( i + 1 \right) \uuu \left( i + 1 \right) \notag \\
\leq 
\mathcal{L} \left( i \right) +
\lambda_1 \www^\top \left( i \right) \mathbf{D}^w \left( i \right) \www \left( i \right) + \notag \\
\lambda_2 \uuu^\top \left( i \right) \mathbf{D}^u \left( i \right) \uuu \left( i \right)
\end{eqnarray}

We then substitute in the definitions for $\mathbf{D}^w$ and $\mathbf{D}^u$:
\begin{eqnarray}
\mathcal{L} \left( i + 1 \right) + 
\lambda_1 \frac{\| \www \left( i + 1 \right) \|_2^2}{2 \| \www \left( i \right) \|_1} + \notag \\
\lambda_2 \sum_{i=1}^{m} \frac{\| \uuu^i \left( i + 1 \right) \|_2^2}{2 \| \uuu^i \left( i \right) \|_2} \notag \\
\leq 
\mathcal{L} \left( i \right) +
\lambda_1 \frac{\| \www \left( i \right) \|_2^2}{2 \| \www \left( i \right) \|_1}
 + \notag \\
\lambda_2 \sum_{i=1}^{m} \frac{\| \uuu^i \left( i \right) \|_2^2}{2 \| \uuu^i \left( i \right) \|_2}
\label{eq:sub}
\end{eqnarray}

From our earlier definitions in Lemma 1, we can define
\begin{eqnarray}
\| \www \left( i + 1 \right) \|_1 - \frac{\| \www \left( i + 1 \right) \|_2^2}{2 \| \www \left( i \right) \|_1} \leq \notag \\
\| \www \left( i \right) \|_1 - \frac{\| \www \left( i \right) \|_2^2}{2 \| \www \left( i + 1 \right) \|_1} \\
\sum_{i=1}^{m} \| \uuu^i \left( i + 1 \right) \|_2 - \sum_{i=1}^{m} \frac{\| \uuu^i \left( i + 1 \right) \|_2^2}{2 \| \uuu^i \left( i \right) \|_2} \leq \notag \\
\sum_{i=1}^{m} \| \uuu^i \left( i \right) \|_2 - \sum_{i=1}^{m} \frac{\| \uuu^i \left( i \right) \|_2^2}{2 \| \uuu^i \left( i \right) \|_2}
\end{eqnarray}

By add these two to Eq. (\ref{eq:sub}), we get
\begin{eqnarray}
\mathcal{L} \left( i + 1 \right) + 
\lambda_1 \| \www \left( i + 1 \right) \|_1 + 
\lambda_2 \sum_{i=1}^{m} \| \uuu^i \left( i + 1 \right) \|_2 \notag \\
\leq
\mathcal{L} \left( i \right) + 
\lambda_1 \| \www \left( i \right) \|_1 + 
\lambda_2 \sum_{i=1}^{m} \| \uuu^i \left( i \right) \|_2 
\end{eqnarray}
showing that the objective value decreases at each iteration.
As our formulation is convex and is lower bounded by zero,
our solution algorithm converges to an optimal solution.

\bibliographystyle{ieeetr}
\bibliography{references}

\end{document}